\documentclass[12pt]{article}

\usepackage{amsmath,amssymb,amsthm,natbib,graphicx,url}
\usepackage[algo2e,ruled]{algorithm2e}
\usepackage{times}
\usepackage{bm}
\usepackage{fullpage}

\newtheorem{theorem}{Theorem}
\newtheorem{lemma}[theorem]{Lemma}
\newtheorem{corollary}[theorem]{Corollary}
\newtheorem{proposition}[theorem]{Proposition}

\let\top\intercal %
\usepackage{notation}

\DeclareBoldMathCommand{\u}{u}
\DeclareBoldMathCommand{\w}{w}
\DeclareBoldMathCommand{\x}{x}
\DeclareBoldMathCommand{\X}{X}
\DeclareBoldMathCommand{\grad}{g}

\newcommand{\RR}{\mathbb{R}}

\newcommand{\cS}{\mathcal{S}}
\newcommand{\cI}{\mathcal{I}}  %

\newcommand{\passed}{\mathcal{P}}         %
\newcommand{\filtered}{\mathcal{F}}       %
\newcommand{\gradlist}{\mathcal{L}}   %

\newcommand{\domain}{\mathcal{W}}

\newcommand{\Regn}{R}
\newcommand{\linReg}{\widetilde{\Regn}}  %

\newcommand{\Norm}[1]{\left\| #1 \right \|}

\newcommand{\intersect}{\cap}

\newcommand{\approxG}{\widetilde{G}}

\DeclareMathOperator*{\pr}{\mathbb P}
\DeclareMathOperator*{\ex}{\mathbb E}
\DeclareMathOperator{\LCB}{LCB}
\DeclareMathOperator{\Bernoullidist}{Bernoulli}
\DeclareMathOperator{\risk}{Risk}
\newcommand{\trisk}{\widetilde \risk}     %
\DeclareMathOperator{\argmin}{arg\,min}
\newcommand{\ind}{\mathbf 1}

  {\list{}{\leftmargin=0.0in\rightmargin=0.0in}  \item[]  }%
  {\endlist}

\title{Robust Online Convex Optimization\\in the Presence of Outliers}

\author{%
\begin{minipage}{\textwidth}
\vspace{.5\baselineskip}
\noindent
\textbf{Tim van Erven} \hfill{} \texttt{tim@timvanerven.nl}\\
\textbf{Sarah Sachs} \hfill{} \texttt{s.c.sachs@uva.nl}\\
\emph{University of Amsterdam, The Netherlands}\\[.5\baselineskip]
\textbf{Wouter M. Koolen} \hfill{} \texttt{wmkoolen@cwi.nl}\\
\emph{Centrum Wiskunde \& Informatica, The
Netherlands}\\[.5\baselineskip]
\textbf{Wojciech Kot{\l}owski} \hfill{} \texttt{wkotlowski@cs.put.poznan.pl}\\
\emph{Pozna{\'n} University of Technology, Poland}\\[.5\baselineskip]
\end{minipage}
}

\begin{document}

\maketitle

\begin{abstract}%
  We consider online convex optimization when a number $k$ of data
  points are outliers that may be corrupted. We model this by
  introducing the notion of robust regret, which measures the regret
  only on rounds that are not outliers. The aim for the learner is to
  achieve small robust regret, without knowing where the outliers are.
  If the outliers are chosen adversarially, we show that a simple
  filtering strategy on extreme gradients incurs $O(k)$ additive overhead
  compared to the usual regret bounds, and that this is unimprovable,
  which means that $k$ needs to be sublinear in the number of rounds. We
  further ask which additional assumptions would allow for a linear
  number of outliers. It turns out that the usual benign cases of
  independently, identically distributed (i.i.d.) observations or
  strongly convex losses are not sufficient. However, combining i.i.d.\
  observations with the assumption that outliers are those observations
  that are in an extreme quantile of the distribution, does lead to
  sublinear robust regret, even though the expected number of outliers
  is linear.
\end{abstract}

\paragraph{keywords}%
  Online convex optimization, robustness, outliers

\section{Introduction}
  Methods for online convex optimization (OCO) are designed to work even in
  the presence of adversarially generated data
  \citep{Hazan2016,ShalevShwartz,CesaBianchiLugosi2006}, but this is
  only possible because strong boundedness assumptions are imposed on
  the losses that limit the influence of individual data points. On the
  other hand, the most practically successful methods do not enforce an
  a priori specified bound on the losses, but instead adapt to the norms
  of the observed gradients or to the observed loss range. For example,
  the regret bound for AdaGrad adapts to the ranges of the gradient
  components per dimension \citep{JMLR:v12:duchi11a}, the regret bound for
  online ridge regression scales with the largest observed loss
  \citep{Vovk01}, the regret bound for AdaHedge in the prediction with
  experts setting scales with the observed loss range of the experts
  \Citep{DeRooijEtAl}, the regret bound for online gradient descent on
  strongly convex losses scales with the maximum gradient norm squared
  \citep{HazanAgarwalKale2007}, etc. In all such cases a small number of
  outliers with large gradients among an otherwise benign dataset can
  significantly worsen performance. This is also clear directly from the
  algorithms themselves, where we see that large gradients have the
  effect of significantly decreasing the effective step size for all
  subsequent data points, leading to slower learning. Extreme outliers
  may occur naturally, for instance because of heavy-tailed
  distributions or sensor glitches, but if each loss is based on the
  input of a user, then we may also be concerned that a small number of
  adversarial users may try to poison the data stream
  \citep{DBLP:conf/iclr/KurakinGB17a}.

  We formally capture the robustness of OCO methods by modifying the
  standard setting to measure performance only on the rounds that are
  not outliers. The goal of the learner is to perform as well as if the
  outliers were not present, up to some overhead that is incurred for
  filtering out the outliers. As in standard OCO, learning proceeds in
  $T$ rounds, and at the start of each round $t$ the learner needs to
  issue a prediction $\w_t$ from a bounded convex domain. The
  environment then reveals a convex loss function $f_t$ with
  (sub)gradient $\grad_t := \nabla f_t(\w_t)$ at $\w_t$, and performance
  is measured by the cumulative difference between the learner's losses
  and the losses of the best fixed parameters $\u$. Unlike in the
  standard OCO setting, however, we only sum up losses over the subset
  of inlier rounds $\cS \subseteq \{1,2,\ldots,T\}$ that are not
  outliers, leading to the following notion of \emph{robust regret}:
  \begin{equation}\label{eqn:robustregret}
    R_T(\u,\cS) := \sum_{t \in \cS} \big(f_t(\w_t) - f_t(\u)\big).
  \end{equation}
  The challenge for the learner is to guarantee small robust regret
  without knowing $\cS$. Importantly, we aim for robust regret bounds
  that scale with the loss range or gradient norms of the rounds
  in~$\cS$, but not with the size of the outliers, so even extreme
  outliers should not be able to confuse the learner.

  In Section~\ref{sec:okbound}, we first consider the fully adversarial
  case where the only thing the learner knows is that there are at most
  $k$ outliers, so
              $T - |\cS| \leq k$,
  and both the inliers and the outliers are generated adversarially,
  without any bound on the range of the outliers, and with the range of
  the inliers also unknown a priori. We introduce a simple filtering
  approach that filters out some of the largest gradients, and passes on
  the remaining rounds to a standard online learning algorithm ALG. When
  the losses are linear, this approach is able to guarantee that
  \begin{equation}\label{eqn:introOKbound}
    R_T(\u,\cS) = R_T^\text{ALG}(\u) + O\big(G(\cS) k\big)
    \qquad
    \text{for all $\cS$ such that $T - |\cS| \leq k$ simultaneously,}
  \end{equation}
  where $G(\cS)$ is the norm of the largest gradient among the rounds
  in $\cS$ and $R_T^\text{ALG}(\u)$ is the regret of ALG on a subset of
  rounds under the guarantee that their gradient norms are at most $2
  G(\cS)$. The extension to general convex losses then follows from a
  standard reduction to the linear loss case. We follow up by showing
  that \eqref{eqn:introOKbound} is unimprovable, not just for
  adversarial losses, but even if the losses are independent and
  identically distributed (i.i.d.) according to a fixed probability
  distribution or if the losses are strongly convex. This fixes the
  dependence on the number of outliers $k$ to be linear in~$k$ in quite some
  generality. Nevertheless, in Section~\ref{sec:quantileMethod} we identify
  sufficient conditions to get around the linear dependence: if the gradients
  are i.i.d., and we take $\cS = \cS_p$ to be the rounds in which 
  $\|\grad_t\|_*$ is at most the $p$-quantile $G_p$ of the common distribution of
  their norms, then there exists a method based on approximating $G_p$ by its
  empirical counterpart on the available data that guarantees that the expected 
  robust regret is at most
  \begin{equation}\label{eqn:introquantilebound}
    \ex \sbr*{R_T(\u,\cS_p)} = O\del*{G_p \del*{\sqrt{p T} + \sqrt{p(1-p)T \ln T} +
    \ln^2 T}}.
  \end{equation}
  Since $O\del*{G_p \sqrt{p T}}$ would be expected if $\cS_p$ were known
  in advance, we see that the overhead grows sublinearly in $T$ and is
  even asymptotically negligible for outlier proportion $1-p = o(1/\ln(T))$.
  More generally, we extend
  this result such that the gradients do not need to be i.i.d.\
  themselves, but it is sufficient if there exist i.i.d.\ random
  variables $\X_t$ and a constant $L$ such that $\|\grad_t\|_* \leq L
  \|\X_t\|_*$. We then define the quantile with respect to the distribution of
  the $\X_t$. This covers nonlinear losses of the form $f_t(\w) = h_t(\w^\top
  \X_t)$ for convex functions $h_t$ that are $L$-Lipschitz, like the logistic
  loss $f_t(\w) = \ln(1+\exp(-Y_t \w^\top \X_t))$ and the hinge loss $f_t(\w) =
  \max\{1-Y_t\w^\top \X_t\}$ for $Y_t \in \{-1,+1\}$, both with $L=1$.

 \paragraph{Related Work}

The definition of robust regret may remind the reader of the adaptive
regret \citep{SheshradiHazan,pmlr-v37-daniely15}, which measures regret on a contiguous
interval of rounds~$\cI$ that is unknown to the learner. Since adaptive
regret can be controlled by casting it into the framework of specialist
(sleeping) experts \citep{fssw-se-97,ChernovVovk2009}, it is natural to
ask whether the same is possible for the robust regret. To apply the
specialist experts framework, we would assign a separate learner
(specialist) to each possible subset of rounds $\cS$ that would then be
active only on $\cS$, and such a pool of $m$ learners would be
aggregated using a meta-algorithm. Computational issues aside, this
approach runs into two problems: the first is that all existing
meta-algorithms assume the losses to be bounded within a known range,
and therefore cannot be applied since we do not assume that even the
range of the inliers is known. Second, even if the range issue could be
resolved, the specialist regret would incur a $\Omega(\sqrt{T \log m}) =
\Omega(\sqrt{k T \log(T/k)}$ overhead, already if we only consider all
$m = \binom{T}{T-k} \geq (T/k)^k$ possible subsets with exactly $k$
outliers. We see that $k$ now multiplies $T$, which is much worse than
the optimal additive dependence on $k$ in~\eqref{eqn:introOKbound}.

Reducing the dependence on the largest gradient norm has previously been considered in
the context of adaptive online and stochastic convex optimization \citep{JMLR:v12:duchi11a, pmlr-v97-ward19a}. However, these methods still depend on the average of all (squared) gradient norms, and therefore require these norms to be finite. In contrast to these adaptive methods, our method can handle a small number of adversarial samples, with large or even infinite norm, while our robust regret analysis still guarantees a sub-linear bound.  

In the context of stochastic optimization, \citet{juditsky2019algorithms} propose a robust version of mirror descent based on truncating the gradients returned by a stochastic oracle. Their main goal is to establish a sub-Gaussian confidence bound on the optimization error under weak assumptions about the tails of the noise distribution. Contrary to our setup, they control the smoothness of the objective and the variance of the noise, so that already a vanilla (non-robust) version of SGD would achieve a vanishing optimization error in expectation (but not with a sub-Gaussian confidence). \citet{pmlr-v97-diakonikolas19a} propose a robust meta-algorithm for stochastic optimization that repeatedly trains a standard algorithm as a base learner and filters out the outliers. This approach is conceptually similar to our filtering method, but it is designed to work in a batch setting, with the data (sample functions) given in advance. \citet{DBLP:journals/corr/abs-1802-06485} provide a robust batch algorithm for stochastic optimization by applying the ideas from robust mean estimation to robustify stochastic gradient estimates in a (batch) gradient descent algorithm.

In the online learning and bandit literature, interesting results were
obtained for dealing with adversarial corruptions of data that are
otherwise generated i.i.d., to still benefit from the stochastic setting
\citep{DBLP:conf/stoc/LykourisML18,pmlr-v99-gupta19a,DBLP:conf/nips/AmirAKML20}.
\citet{wang2018data} and \citet{pmlr-v120-zhang20b} further consider data poisoning attacks on an online learner, but the focus is on the optimization of the adversary, while the learner remains fixed.
In all these works, contrary to ours, the corrupted data is still
assumed to lie in the
same range as the non-corrupted data. A notable exception is the very
recent work of \citet{DBLP:journals/corr/abs-2010-04157}, which proposes
online algorithms for contextual bandits and linear regression in a
framework in which the linear model is realizable (well-specified) up to
small noise, and a fixed, randomly selected, fraction of examples is
arbitrarily corrupted (as in the Huber $\epsilon$-contamination model
\citep{Huber}), but still remains bounded. In contrast, we avoid strong
distributional assumptions such as model realizability, and do not make
any probabilistic assumptions about the corruption mechanism or impose
any constraints on the magnitude of the outliers.

Starting with pioneering works of Tukey and Huber \citep{Tukey,Huber} there has been a tremendous amount of past work in the area of robust statistics, which concerns the basic tasks of classical statistics in the presence of outliers and heavy-tailed distributions \citep{HuberBook}. A more recent line of research building on the work of \citet{Catoni2012,Minsker,LugosiMendelson2019,LugosiMendelson2019survey} concerns estimation with sub-Gaussian-style confidence for heavy-tailed distributions. 
Finally, our setup is different from, but conceptually related to, a
line of research on machine learning and statistical problems in the
presence of adversarial data corruptions \citep{Charikar_etal}. This has
been studied, for instance, in the context of parameter estimation
\citep{Lai_etal,10.5555/3310435.3310606,10.5555/3174304.3175475,LugosiMendelson2019adversarial,pmlr-v108-prasad20a},
robust PCA \citep{robustPCA}, regression
\citep{pmlr-v75-klivans18a,Diakonikolas_etal,pmlr-v108-liu20b},
classification
\citep{JMLR:v10:klivans09a,10.1109/TPAMI.2015.2456899,10.1145/3006384}
and many other cases. See the in-depth survey by \citet{DBLP:journals/corr/abs-1911-05911} for an overview of recent advances in this direction.

\paragraph{Outline}
 
We start by summarizing our setting and notation in the next section.
Then, in Section~\ref{sec:okbound}, we prove the upper bound
\eqref{eqn:introOKbound} for adversarial losses, and show matching lower
bounds both for i.i.d.\ losses and for strongly convex losses. As a
further example, we show how robust regret can be used to bound the
excess risk in the Huber $\epsilon$-contaminated setting via
online-to-batch-conversion.
 In
Section~\ref{sec:quantileMethod} we turn to the quantile case and establish
\eqref{eqn:introquantilebound}.
Finally, Section~\ref{sec:conclusion} concludes with a discussion of possible
directions for future work. Some proofs are deferred to the
appendix.
 
\section{Setting and Notation} \label{sec:setting}

Formally, we consider the following online learning protocol. In each
round $t = 1,2,\ldots$ the learner first predicts $\w_t \in \domain$,
where the domain $\domain$ is a non-empty, compact and convex subset of
$\RR^d$. The adversary then reveals a convex loss function $f_t: \domain
\rightarrow \RR$, and the learner suffers loss $f_t(\w_t)$. We assume
throughout that there always exists a gradient or, more generally, a
subgradient $\grad_t := \nabla f_t(\w_t)$ at the learner's prediction,
which is implied by convexity of $f_t$ whenever $\w_t$ lies in the
interior of $\domain$ and also on the boundary if there exists a finite
convex extension of all $f_t$ to a larger domain that contains $\domain$
in its interior. The performance of the learner with respect to any
fixed parameters $\u \in \domain$ is measured by the \emph{robust
regret} $R_T(\u,\cS)$ over the rounds $\cS \subseteq [T] :=
\{1,2,\ldots,T\}$ that are not outliers, as defined in
\eqref{eqn:robustregret}. The definition of subgradients implies that
$f_t(\w_t) - f_t(\u) \leq (\w_t - \u)^\top \grad_t$, which implies that
$R_T(\u,\cS)$ is bounded from above by the \emph{linearized robust
regret}
\[
  \linReg_T(\u,\cS) := \sum_{t \in \cS} (\w_t - \u)^\top \grad_t.
\]
We will state our main results for an arbitrary norm $\Norm{\cdot}$ on
$\domain$ and measure gradient lengths in terms of the dual norm
$\Norm{\grad_t}_* = \sup_{\w \in \mathbb R^d : \Norm{\w} \le 1} \w^\top \grad_t$. Let $D = \max_{\u,\w \in \domain} \Norm{\w - \u}$
denote the diameter of the domain. For the analysis of the robust
regret, we need a Lipschitz bound for the gradients that are in the set
$\cS$, which we denote by
\begin{equation*}
  G(\cS) := \max_{t \in \cS} \Norm{\grad_t}_*.
\end{equation*}

\section{Robustness to Adversarial Outliers}
\label{sec:okbound}

In this section we derive matching upper and lower bounds of the form in
\eqref{eqn:introOKbound}.

\subsection{Upper Bounds}

Let ALG be any Lipschitz-adaptive algorithm, which we will use as our
base online learning algorithm. Our general approach is to add a
filtering meta-algorithm FILTER that examines (the norm of) incoming
gradients and decides whether to filter them or pass them on to ALG for
learning. If $\cS$ were known in advance, then FILTER could filter out
all outliers and pass on only the rounds in $\cS$, but since $\cS$ is
not known, FILTER needs to learn which rounds to pass on. Although most
online learning algorithms base their updates only on gradients, we note
that we do allow ALG to use the full loss function $f_t$ to update its
state when FILTER passes on round $t$ to ALG. When a round $t$ is
filtered, we assume that ALG behaves as if that round had not happened,
so we will have $\w_{t+1} = \w_t$. Our FILTER for this section is displayed in Algorithm~\ref{alg:filterok}.

\begin{algorithm2e}[htb]
\caption{Top-$k$ Filter: Filtering for Adversarial Setting}
\KwIn{Maximum number of outliers $k$}
\textbf{Initialize:} Let $\gradlist_0 = \{0,0,\ldots,0\}$ be an ordered
list of length $k+1$.\\
\For{$t = 1,2,\ldots$}{
  \CommentSty{Maintain invariant that $\gradlist_t$ contains $k+1$ largest
  gradients}\;
  \eIf{$\|\grad_t\|_* > \min \gradlist_{t-1}$}{
    Obtain $\gradlist_t$ from $\gradlist_{t-1}$ by removing the smallest
    item in $\gradlist_{t-1}$ and inserting $\|\grad_t\|_*$\;
  }{
    Set $\gradlist_t$ equal to $\gradlist_{t-1}$\;
  }
  \CommentSty{Filter with factor $2$ slack}\;
  \eIf{$\|\grad_t\|_* > 2 \min \gradlist_t$}{
    Filter round $t$\;
  }{
    Pass round $t$ on to ALG\;
  }
}
\label{alg:filterok}
\end{algorithm2e}

\begin{theorem}\label{thm:okbound}
  Suppose ALG is any Lipschitz-adaptive algorithm that guarantees
  linearized regret bounded by $B_T(G)$ on the rounds that it is passed
  by FILTER, if the gradients in those rounds have length at most $G$,
  and let FILTER be Algorithm~\ref{alg:filterok} with parameter $k$.
  Then the linearized robust regret of ALG+FILTER is bounded by
  \begin{equation}\label{eqn:okbound}
    \linReg_T(\u,\cS)
      \leq B_T\big(2 G(\cS)\big) + 4D(\u,\cS) G(\cS) (k+1)
    \qquad
    \text{for any $\cS: T - |\cS| \leq k$ and $\u \in \domain$,}
  \end{equation}
  where $D(\u,\cS) = \max_{t : \|\grad_t\|_* \leq 2 G(\cS)} \|\w_t - \u\|$.
\end{theorem}
There are two main ideas to the proof. First, since the list
$\gradlist_t$ in Algorithm~\ref{alg:filterok} contains $k+1$ elements
and there are at most $k$ outliers, at least one of the elements of
$\gradlist_t$ must be one of the inliers from~$\cS$. It follows that the
smallest element of $\gradlist_t$ is a lower bound on $G(\cS)$. The
second idea is that, instead of filtering on this lower bound directly,
we filter with factor $2$ slack. Since every filtered gradient is also
added to $\gradlist_t$, this factor $2$ ensures that the minimum of
$\gradlist_t$ must at least double for every $k+1$ rounds that are
filtered. The resulting exponential growth of the filtered rounds means
that the contribution to the robust regret of all filtered rounds is
dominated by the last $k+1$ rounds, and therefore does not grow
with~$T$.

\begin{proof}%
  Let $\filtered \subset [T]$ denote the rounds filtered out by
  Algorithm~\ref{alg:filterok}, and let $\passed = [T] \setminus
  \filtered$ denote the rounds that are passed on to ALG. Then the
  linearized robust regret splits as follows:
  \begin{equation*}
    \linReg_T(\u,\cS)
      = \sum_{t \in \cS \intersect \passed} (\w_t - \u)^\top \grad_t
        + \sum_{t \in \cS \intersect \filtered} (\w_t - \u)^\top
          \grad_t.
  \end{equation*}
  We will show that Algorithm~\ref{alg:filterok} guarantees that the
  gradients on the passed rounds are bounded as follows:
  \begin{equation}\label{eqn:passedgradientsbound}
    \|\grad_t\|_* \leq 2 G(\cS)
    \qquad
    \text{for all $t \in \passed$,}
  \end{equation}
  which implies that
  \begin{align*}
    \sum_{t \in \cS \intersect \passed} (\w_t - \u)^\top \grad_t 
    &= \sum_{t \in \passed} (\w_t - \u)^\top \grad_t
      - \sum_{t \in \passed \setminus \cS} (\w_t - \u)^\top \grad_t\\
    &\leq B_T(2 G(\cS)) + 2 D(\u,\cS) G(\cS) |\passed \setminus \cS|
    \leq B_T(2 G(\cS)) + 2 D(\u,\cS) G(\cS) k,
  \end{align*}
  where the first inequality uses the assumption on ALG and H\"older's
  inequality, and the second inequality uses that $|\passed \setminus
  \cS| \leq |[T] \setminus \cS| \leq k$.

  We proceed to prove \eqref{eqn:passedgradientsbound}. During the first
  $k$ rounds, $\min \gradlist_t = 0$, so
  \eqref{eqn:passedgradientsbound} is trivially satisfied. In all later
  rounds, $\gradlist_t \subseteq \{\|\grad_s\|_* : s \leq t\}
  \subseteq \{\|\grad_s\|_* : s \leq T\}$. Consequently, $\gradlist_t$
  must contain at least one element $\|\grad_t\|_*$ with $t \in \cS$,
  because $T - |\cS| \leq k$ and $|\gradlist_t| = k+1 > k$.
  It follows that $\min \gradlist_t \leq G(\cS)$, so all passed
  gradients satisfy \eqref{eqn:passedgradientsbound}.

  Let $G_\text{min} = \min \{\|\grad_t\|_* \mid t \in \filtered\} > 0$ be
  the length of the shortest filtered gradient. To complete the proof,
  we will show that
  \[
    \sum_{t \in \cS \intersect \filtered} (\w_t - \u)^\top \grad_t
      \leq D(\u,\cS) \!\sum_{t \in \cS \intersect \filtered} \!\!\|\grad_t\|_*
      \leq D(\u,\cS) \hspace{-2em}\sum_{\substack{t \in \filtered \\ G_\text{min} \leq \|\grad_t\|_* \leq G(\cS)}} \hspace{-2em} \|\grad_t\|_*
      \leq 2D(\u,\cS) G(\cS) (k+1).
  \]
  The first of these inequalities follows from H\"older's inequality, and the second
  from the definition of $G(\cS)$. To establish the last inequality, we
  proceed by induction: since $\gradlist_t$ contains the $k+1$ largest observed
  gradient norms, we observe that there can be at most $k+1$ filtered rounds in
  which $G(\cS)/2^{i+1} < \|\grad_t\|_* \leq G(\cS)/2^i$, because after $k+1$
  such rounds we will have $\min \gradlist_t > G(\cS)/2^{i+1}$ forever. It
  follows that we have the following induction step:
  \[
    \hspace{-2em}\sum_{\substack{t \in \filtered \\ G_\text{min} \leq
    \|\grad_t\|_* \leq G(\cS)/2^i}} \hspace{-2em} \|\grad_t\|_*
      \leq (k+1) G(\cS)/2^i \quad +
      \hspace{-2em}\sum_{\substack{t \in \filtered \\ G_\text{min} \leq
      \|\grad_t\|_* \leq G(\cS)/2^{i+1}}} \hspace{-2em} \|\grad_t\|_*
      \qquad \text{for $i=0,1,2,\ldots$}
  \]
  Unrolling the induction, we therefore obtain
  \[
    \hspace{-2em}\sum_{\substack{t \in \filtered \\ G_\text{min} \leq \|\grad_t\|_* \leq G(\cS)}} \hspace{-2em} \|\grad_t\|_*
      \leq (k+1) G(\cS) \!\!\sum_{i=0}^{\ceil{\log_2
      \frac{G(\cS)}{G_{\text{min}}}}} \!\!2^{-i}
      \leq (k+1) G(\cS) \sum_{i=0}^\infty 2^{-i}
      = (k+1) G(\cS) 2,
   \]
  which is what remained to be shown.
\end{proof}

As for the run-time, one may maintain the $k$ largest gradient norms encountered in a priority queue. The time used by Algorithm~\ref{alg:filterok} on top of ALG is $O(\ln k) \le O(\ln T)$ per round. This may be pessimistic in practise, as FILTER only performs work if the current gradient is among the $k+1$ largest seen so far.

\subsubsection{Examples}\label{sec:examples}

To make the result from Theorem~\ref{thm:okbound} more concrete, let us
instantiate ALG as online gradient descent (OGD), which starts from any
$\w_1 \in \domain$ and updates according to
\[
  \w_{t+1} = \Pi_\domain(\w_t - \eta_t \grad_t),
\]
where $\Pi_\domain(\w)$ denotes Euclidean projection of $\w$ onto
$\domain$, and $\eta_t > 0$ is a hyperparameter called the step size.
Tuning the step size for general convex losses, we find that we can
tolerate at most $k = O(\sqrt{T})$ outliers without suffering in the
rate:
\begin{corollary}[General Convex Losses]\label{cor:generalconvex}
  Let $\Norm{\cdot}$ be the $\ell_2$-norm, let ALG be OGD with step size
  $\eta_t = D/\sqrt{2 \sum_{s=1}^t \|\grad_s\|_2^2}$ and let FILTER be
  Algorithm~\ref{alg:filterok} with parameter $k$. Then the robust regret is
  bounded by
  \begin{equation}\label{eqn:generalconvexlosses}
  \begin{split}
    R_T(\u,\cS)
      &\leq 2 D \sqrt{\sum_{t \in \cS} \|\grad_t\|_2^2}
           + 2D G(\cS) \Big(2k + \sqrt{k} + 2\Big)\\
      &\leq 2 D G(\cS) \Big(\sqrt{T} + 2k + \sqrt{k} + 2\Big)
    \qquad
    \text{for any $\cS: T - |\cS| \leq k$ and $\u \in \domain$.}
  \end{split}
  \end{equation}
\end{corollary}
The proof of the corollary is in Appendix~\ref{app:exampleproofs}.

The step size of OGD may also be tuned for \emph{$\sigma$-strongly
convex losses}, which are guaranteed to be curved in all directions, and
satisfy the requirement that
\[
  f_t(\u) \geq f_t(\w) + (\u - \w)^\top \nabla f_t(\w) +
  \frac{\sigma}{2} \|\u - \w\|_2^2 
  \qquad \text{for all $\u,\w \in \domain$.}
\]
In this case, we obtain the following guarantee on the robust regret, which is proved in Appendix~\ref{app:exampleproofs}:
\begin{corollary}[Strongly Convex Losses]
  \label{cor:stronglyconvex}
  Suppose the loss functions $f_t$ are $\sigma$-strongly convex. Let
  $\Norm{\cdot}$ be the $\ell_2$-norm, let ALG be OGD with step size $\eta_t =
  \frac{1}{\sigma t}$ and let FILTER be Algorithm~\ref{alg:filterok} with
  parameter $k$. Then the robust regret is bounded by
  \begin{equation}
    R_T(\u,\cS)
      \leq \frac{2 G(\cS)^2}{\sigma} \big(\ln T + 1\big)
        + \frac{5\approxG(\u,\cS)^2}{2\sigma}(k+1)
    \qquad
    \text{for any $\cS: T - |\cS| \leq k$ and $\u \in \domain$,}
  \end{equation}
  where $\approxG(\u,\cS) = 2 G(\cS) + \max_{t : \|\grad_t\|_2 \leq 2G(\cS)} \|\nabla
  f_t(\u)\|_2$.
\end{corollary}
The standard regret bound of OGD for strongly convex losses is of order
$\frac{G^2}{\sigma} \log T$ \citep{HazanAgarwalKale2007}, so in this
case we can tolerate $k = O(\log T)$ outliers without suffering in the
rate, under the additional assumption that $\approxG(\u,\cS) =
O(G(\cS))$. This seems like a reasonable assumption if we think of the
condition $\|\grad_t\|_2 \leq 2 G(\cS)$ as expressing that round $t$ is
not too extreme.
 
\paragraph{Huber $\epsilon$-Contamination}

As a final example, we consider the Huber $\epsilon$-contamination
setting \citep{Huber}. In this case losses are of the form $f_t(\w) =
f(\w,\xi_t)$, where the random variables $\xi_t$ are sampled i.i.d.\
from a mixture distribution $P_\epsilon$ defined by
\[
  \xi \sim \begin{cases}
    P & \text{if $M = 0$}\\
    Q & \text{if $M = 1$}
  \end{cases}
  \qquad \text{where} \qquad M \sim \Bernoullidist(\epsilon)
\]
for some $\epsilon \in [0,1)$. The interpretation is that $P$ is the
actual distribution of interest, which is contaminated by outliers drawn
from $Q$. The hidden variable $M$ is not observed by the learner, so it
is not known which observations are outliers. Let $\cS^* \subseteq [T]$
denote the set of inlier rounds in which $M_t = 0$. Then the robust
regret $R_T(\u,\cS^*)$ may be viewed as the ordinary regret on a modified
loss function $\tilde f(\w,M,\xi)$ that is equal to $f(\w,\xi)$ on
samples from $P$ but zero on samples from~$Q$, i.e.\ $\tilde f(\w,M,\xi)
= \mathbf{1}\{M = 0\}f(\w,\xi)$ and
\begin{equation}\label{eqn:robustasregularregret}
  R_T(\u,\cS^*) = \sum_{t=1}^T \del*{\tilde f(\w_t,M_t,\xi_t) - \tilde
  f(\u,M_t,\xi_t)}.
\end{equation}
Let the risk with respect to the inlier distribution $P$ be defined as
\[
  \risk_P(\w) = \ex_{\xi \sim P} \sbr*{f(\w,\xi)}.
\]
Then, applying online-to-batch conversion
\citep{CesaBianchiConconiGentile2004} to the modified loss $\tilde f$,
we obtain the following result, which bounds the excess risk under $P$
by the robust regret when the observations are drawn from the
contaminated mixture $P_\epsilon$, without requiring any assumptions
about the outliers coming from~$Q$:
\begin{lemma}[Huber $\epsilon$-Contamination]\label{lem:huber}
  Suppose the losses $f_t$ are i.i.d.\ according to the mixture
  distribution $P_\epsilon$, and let $\u_P \in \argmin_{\w \in \domain}
  \risk_P(\w)$ be the optimal parameters for the distribution of the
  inliers. Let $\bar \w_T = \frac{1}{T} \sum_{t=1}^T \w_t$, where
  $\w_1,\ldots,\w_T$ are the predictions of the learner. Then
  \begin{equation}\label{eqn:huberExp}
    \ex_{P_\epsilon} \sbr*{\risk_P(\bar \w_T) - \risk_P(\u_P)}
      \leq \frac{\ex_{P_\epsilon} \sbr*{R_T(\u_P,\cS^*)}}{(1-\epsilon)T}.
  \end{equation}
  Moreover, if $|f(\w,\xi) - f(\u_P,\xi)| \leq B$ almost surely when $\xi \sim
  P$ is an inlier, then for any $0 < \delta \leq 1$
  \begin{equation}\label{eqn:huberProb}
    \risk_P(\bar \w_T) - \risk_P(\u_P)
      \leq \frac{R_T(\u_P,\cS^*)}{(1-\epsilon)T}
      + \frac{2B}{1-\epsilon} \sqrt{\frac{2}{T} \ln \frac{1}{\delta}}
  \end{equation}
  with $P_\epsilon$-probability at least $1-\delta$.
\end{lemma}
(Details of the proof are given in Appendix~\ref{app:exampleproofs}.) We
see that, if we can control the robust regret with respect to the
unknown set $\cS^*$ of inlier rounds, then we can also control the excess
risk with respect to the inlier distribution $P$. For example, 
instantiating the learner as in Corollary~\ref{cor:generalconvex}
leads to the following specialization of Lemma~\ref{lem:huber}.
\begin{corollary}\label{cor:hubercor}
  In the setting of Lemma~\ref{lem:huber}, suppose that $\|\nabla
  f(\w,\xi)\| \leq G$ for all $\w \in \domain$ almost surely when $\xi
  \sim P$ is an inlier, and that $\epsilon \leq 1/2$. Let the
  learner be instantiated as in Corollary~\ref{cor:generalconvex} with
  $k = \ceil{\epsilon T + \sqrt{2 T \epsilon (1-\epsilon) \ln(2/\delta)}
  + \tfrac{1}{3}(1-\epsilon) \ln(2/\delta)}$ for any $0 <
  \delta \leq 1$. Then
  \begin{equation}\label{eqn:huberapplied}
    \risk_P(\bar \w_T) - \risk_P(\u_P)
      \leq
      12 D G \epsilon
      + \frac{2DG \big(5 \sqrt{2\ln(2/\delta)}+2\big)}{\sqrt{T}}
      + \frac{2 D G \big(\ln(2/\delta) + 10\big)}{T}
  \end{equation}
  with $P_\epsilon$-probability at least $1-\delta$.
\end{corollary}
Here $\nabla f(\w,\xi)$ should be read as the gradient of $f(\w,\xi)$
with respect to~$\w$. The constant dependence on $DG\epsilon$, which does
not go to zero with increasing $T$, is unavoidable because $P$ is
non-identifiable based on samples from $P_\epsilon$. For instance,
consider the linear loss $f(w,\xi) = \xi w$ with $\domain = [-D/2,+D/2]$
and $P_\epsilon$ such that $\xi = -G$ and $\xi = +G$ both with
probability $\epsilon$, and $\xi = 0$ with probability $1-2\epsilon$.
Then we cannot distinguish the case that $P = P_\epsilon(\cdot \mid \xi
\leq 0)$ and $Q$ is a point-mass on $+G$ from the case that
$P=P_\epsilon(\cdot \mid \xi \geq 0)$ with $Q$ a point-mass on $-G$. No
matter what the output of the learner is, its excess risk under $P$ will
always be at least $DG\epsilon$ in one of these two cases.

The proof of Corollary~\ref{cor:hubercor} is postponed to
Appendix~\ref{app:exampleproofs}. It is a straightforward combination of
Lemma~\ref{lem:huber} and Corollary~\ref{cor:generalconvex}, with the
only point of attention being the tuning of the number of outliers~$k$.
In expectation, the number of outliers is $\epsilon T$, but we choose
$k$ slightly larger so that the probability that the number of outliers
exceeds $k$ is negligible.

\subsection{Lower Bounds}

We now show that the bounds obtained in the previous part of this section are
non-improvable in general. First note that one can always choose $\cS =
[T]$ (no outliers) and apply a standard lower bound for online learning
algorithms which guarantees expected regret $\Omega(\sqrt{T})$ for general
losses and $\Omega(\ln T)$ for strongly-convex losses. This matches the first term in the
bound of Theorem \ref{thm:okbound}. Therefore, we will only
show a bound $\Omega(k)$, which, combined with the standard one, leads to a
$\Omega(\max\{\sqrt{T},k\}) = \Omega(\sqrt{T} + k)$ lower bound on the regret
for general convex losses and $\Omega(\ln T + k)$ for strongly convex losses.

Consider a learning task over domain $\domain = [-W,W]$ for some $W > 0$.
To prove a lower bound for general convex losses, we choose the loss sequence
to be $f_t(w) = G\xi_t w$, where $\xi_t \in \{-1,+1\}$ are i.i.d.\ Rademacher
random variables with $\Pr(\xi_t = -1) = \Pr(\xi_t = +1) = \frac 12$, while $G > 0$
controls the size of the gradients/losses.

\begin{theorem}[Lower Bound with I.I.D.\ Losses]
\label{thm:lower_bound_1}
For any $k$ and any online learning algorithm run on the sequence defined above, there
exist adversarial choices of $\cS$ with $T-|\cS| \le k$ and $u \in \domain$ such that 
\[
\ex_{f_1,\ldots,f_T} \sbr*{ R_T(u, \cS) } \geq \frac{D G(\cS) k}{4},
\]
where $f_1,\ldots,f_T$ are i.i.d.\ as described above.
\end{theorem}
\begin{proof}
Let $S_1 = \{t \in [k] \colon \xi_t = 1\}$ and $S_{-1} = \{t  \in [k] \colon \xi_t = -1\}$. The
adversary will choose $u = -W \zeta$ and $\cS = S_{\zeta} \cup \{k+1,\ldots,T\}$, 
where $\zeta \in \{-1,1\}$ is
a Rademacher random variable independent of $\xi_1,\ldots,\xi_T$. The expected regret
jointly over $\xi_1,\ldots,\xi_T,\zeta$ is then given by
\begin{align*}
\ex \sbr*{ R_T(u, \cS) } &=
\ex \sbr*{ G \sum_{t=1}^T \ind_{t \in \cS} w_t \xi_t - G \sum_{t=1}^T
\ind_{t \in \cS} u \xi_t} \\
&=G \sum_{t=1}^k \underbrace{\ex \sbr*{\ind_{\zeta = \xi_t} w_t \xi_t }}_{=0}
+ G \sum_{t=k+1}^T \underbrace{\ex \sbr*{w_t \xi_t}}_{=0}
+ GW \sum_{t=1}^k \underbrace{\ex \sbr*{\ind_{\zeta = \xi_t} \zeta \xi_t }}_{=1/2}
+ GW \sum_{t=k+1}^T \underbrace{\ex \sbr*{\zeta \xi_t}}_{=0} \\
&= GW \frac{k}{2},
\end{align*}
where we used the independence of $\xi_t$ and $\zeta$ in the second and the fourth sum,
while 
\[
\ex \sbr*{\ind_{\zeta = \xi_t} w_t \xi_t }
= \ex \sbr*{\; \ex \sbr*{\ind_{\zeta = \xi_t} w_t \xi_t \, |\, \xi_t}\;}
= \ex \sbr*{w_t \xi_t/ 2} = 0, \quad \text{and} \;
\ex \sbr*{\ind_{\zeta = \xi_t} \zeta \xi_t } = \ex \sbr*{\ind_{\zeta = \xi_t}} = \frac{1}{2}.
\]
As the bound holds for the random choice of $\zeta$ it also holds for the worst-case choice
of $\zeta$. The theorem now follows from $D = \max_{u,w \in \domain} |w - u| = 2 W$ and $G(\cS) = \max_{t \in \cS} |g_t| = \max_{t \in \cS} G |\xi_t| = G$.
\end{proof}
A similar bounding technique leads to a lower bound for $\sigma$-strongly convex losses,
except that the distribution of the losses differs between the first $k$ rounds
and the later rounds. This still implies a lower bound for adversarially
generated data, but not for i.i.d.\ losses. In this case, we will choose the domain
$\domain = [-W,W]$, the
loss sequence based on the $\sigma$-strongly convex squared loss, $f_t(w) = \frac{\sigma}{2} (w - W \xi_t)^2$, for $t \le
k$, and $f_t(w) = \frac{\sigma}{2} (w-W\zeta)^2$ for $t \ge k$, where $\xi_1,\ldots\xi_k$ and
$\zeta$ are again i.i.d.\ Rademacher variables.
\begin{theorem}[Lower Bound for Strongly Convex Losses]
For any $k$ and any online learning algorithm, there exist adversarial choices of $\cS$ with $T-|\cS| \le k$ and $u \in \domain$ such that
\[
\ex_{f_1,\ldots,f_T} \sbr*{ R_T(u, \cS) } \geq \frac{G^2(\cS) k}{16 \sigma},
\]
where $f_1,\ldots,f_T$ are the %
$\sigma$-strongly convex losses described above.
\end{theorem}
\begin{proof}
Using the same notation as in the proof of Theorem~\ref{thm:lower_bound_1}, the adversary will choose $u = W \zeta$
and $\cS = S_{\zeta} \cup \{t+1,\ldots,T\}$.
The expected regret
jointly over $\xi_1,\ldots,\xi_k,\zeta$ is given by
\begin{align*}
\ex \sbr*{ R_T(u, \cS) } &=
\frac{\sigma}{2} \sum_{t=1}^k \underbrace{\ex \sbr*{\ind_{\zeta = \xi_t} (w_t - W\xi_t)^2}}_{\ge W^2/2}
+ \frac{\sigma}{2} \sum_{t=k+1}^T \underbrace{\ex \sbr*{(w_t - W\zeta)^2}}_{\ge 0} \\
&\quad - \frac{\sigma}{2} \sum_{t=1}^k \underbrace{\ex \sbr*{\ind_{\zeta = \xi_t} (\zeta - W\xi_t)^2}}_{=0}
\ge \frac{\sigma W^2 k}{4},
\end{align*}
where to bound the first sum we used 
\begin{align*}
\ex \sbr*{\ind_{\zeta = \xi_t} (w_t - W\xi_t)^2 }
&= \ex \sbr*{\; \ex \sbr*{\ind_{\zeta = \xi_t} (w_t - W\xi_t)^2 \, |\, \xi_t}\;}
= \ex \sbr*{(w_t - W\xi_t)^2 / 2} \\
&= \ex \sbr*{w_t^2/2 - W\xi_t w_t + W^2/2} = w_t^2/2 + W^2/2 \ge W^2/2.
\end{align*}
To finish the proof note that 
$|\nabla f_t(w_t)| = \sigma|w_t - W\xi_i| \le 2 \sigma W$ so that
$G(\cS) \le 2 \sigma W$.
\end{proof}

\section{Robustness for Quantiles}
\label{sec:quantileMethod}

In this section we consider robust online linear optimization in the stochastic i.i.d.\ setting. That is, we consider i.i.d.\ gradients $\grad_t \sim \pr$ that are in particular independent of the learner's prediction $\w_t$. Let $G_p \df q_p(\norm{\grad}_*)$ be the $p$-quantile of the gradient in dual norm $\norm{\cdot}_*$. To keep things simple, we will assume that $\pr$ does not have an atom at $G_p$, so that $\pr\set{\norm{\grad}_* \le G_p} = p$ exactly. We call a gradient $\grad_t$ an \markdef{outlier} if $\norm{\grad_t}_* > G_p$.
Fix a domain $\domain$ of diameter $D$ in the norm $\norm{\cdot}$. We are interested in algorithms that know $\domain$ and $p$ but not $G_p$, play $\w_t \in \domain$, and we aim to bound their expected robust regret on the (random!) set of inliers $\cS = \set{t \in [T] : \norm{\grad_t}_* \le G_p}$. That is, we aim to control
\begin{equation}\label{eq:exp.rob.reg}
  \bar R_T
  ~\df~
  \ex \sbr*{
    \max_{\u \in \domain}
    R_T(\u, \cS)
  }
  ~=~
  \ex \sbr*{
    \max_{\u \in \domain}
    \sum_{t \in [T] : \norm{\grad_t}_* \le G_p} \tuple{\w_t - \u, \grad_t}
  }
  .
\end{equation}
Note that a bound on the expected robust regret implies a robust pseudo-regret bound, where the data-dependent maximum is replaced by the fixed minimiser of the expected loss on inliers, i.e.\ $\u^* \in \arg\min_{\u \in \domain} \u^\top \ex\sbrc{\grad_t\,}{\,\norm{\grad_t}_* \le G_p}$.
Our FILTER algorithm for the stochastic setting is shown as Algorithm~\ref{alg:filtermetaAlgorithmQuantile}. The main idea is that it only passes rounds to the base ALG for which it is virtually certain that  they are inliers. To this end our FILTER computes a lower confidence bound $\LCB_t$ on the quantile $G_p$. Smaller gradients are included, while larger ones are discarded. The crux of the robust regret bound proof is then dealing with the inlier gradients that end up being dropped. We will find it instructive to state our algorithms and confidence bounds with a free confidence parameter $\delta$. Tuning our approach will then lead us to set $\delta=T^{-2}$.

\begin{algorithm2e}[htb]
\KwIn{Quantile level $p \in (0,1)$, confidence $\delta$, online learner ALG}
\For{$t = 1,2,\ldots$}{
  Have ALG produce $\w_t$. Receive gradient $\grad_t$\;

  Let $\hat q_{t-1}$ be the empirical quantile function of past gradients $\grad_1, \ldots, \grad_{t-1}$. \;

  Compute $\LCB_{t-1} = \hat q_{t-1}(p - u_{t-1})$ at threshold $u_{t-1} = \sqrt{t^{-1} 2 p(1-p) \ln \frac{1}{\delta}}
  + \frac{1}{3} t^{-1} \ln \frac{1}{\delta}$ \;

  \eIf{$\norm{\grad_t}_* \le \LCB_{t-1}$}{
    Pass round $t$ on to ALG\;
  }{
    Ignore round $t$\;
  }
}
\caption{Filtering meta algorithm for Robust Quantile Regret}
\label{alg:filtermetaAlgorithmQuantile}
\end{algorithm2e}

We now show that the expected robust regret is small.

\begin{theorem}
\label{Thm:quantilebound}
  Let ALG have individual sequence regret bound $B_T(G)$ for $T$ rounds with gradients of dual norm at most $G$, and which is concave in $T$. Let $D$ be the diameter of the domain. Then the FILTER Meta-Algorithm~\ref{alg:filtermetaAlgorithmQuantile} with $\delta=T^{-2}$ has expected robust regret bounded by
  \[
    \bar R_T
    ~\le~
    B_{p T}(G_p)
    +
    D G_p \del*{
      4 \sqrt{2 p (1-p) T \ln T}
      + \frac{13}{3} (\ln T)^2
      + 3
    }
    .
  \]
\end{theorem}

If ALG does its job, the first term is the minimax optimal regret for when the outlier rounds were
known. The other terms quantify the cost of being robust. When $p$ is
not extreme, this cost is of order $G_p D \sqrt{T \ln T}$, rendering it
the dominant term overall (escalating the minimax regret by a mild log
factor). When $p$ tends to $1$ or $0$, the robustness overhead gracefully reduces to the $(\ln T)^2$ regime.
\medskip

The proof can be found in Appendix~\ref{appx:proof.quantiles}. The main ideas are as follows. As we have no control over outlier gradients (they may be astronomical), we must assume that ALG gets confused without recourse if FILTER ever passes it any outlier. Note that FILTER is not evaluated on outlier rounds, so it does not suffer from this gradient's \emph{magnitude}. But its effect is that, for all we know, ALG is rendered forever useless, upon which  FILTER may incur the maximum possible regret of $G_p D T$. Our approach will be to choose our threshold for inclusion conservatively, and to apply concentration in all rounds simultaneously, to ensure this bad event is rare (this is the source of the $\ln T$ factor). A second concentration allows us to deal with the discarded inliers.

\medskip

\noindent
We conclude the section with a selection of remarks.

\noindent
\textbf{Examples} The examples of Section~\ref{sec:examples} also apply here. Depending on the setting, and hence the appropriate base algorithm ALG, the dominant regret term can be either the $D G_p \sqrt{p (1-p) T \ln T}$ term, or the $B_{p T}(G_p)$ term. The former case applies for OGD, while the latter case happens in the $K$-experts setting with many experts and few rounds, i.e.\ $K \gg T$. There adding robustness comes essentially for free.

\noindent
\textbf{Anytime Robust Regret} As stated, the algorithm needs to know the horizon $T$ up front to set the confidence parameter $\delta$ in the deviation width $u_t$. We can use a standard doubling trick on $T$ to get an anytime algorithm.

\noindent
\textbf{Anytime concentration} One may wonder how much the analysis can be improved by replacing our union bound over time steps with a time-uniform Bernstein concentration inequality, as e.g.\ developed by \cite{howard2019sequential}. Sadly, the best we can hope for is to be able to use $\delta = \frac{1}{T}$, which would lead to a constant factor $\sqrt{2}$ improvement on the dominant term. We cannot tolerate a higher overall failure probability, for we have to pacify the regret upon failure, which may be of order $T$.

\noindent
\textbf{High Probability Version} Going into the proof, we see that a high probability robust regret bound is also possible. We would need to change the analysis of $P^{(2)}$, as we currently analyse it in expectation. Observing that it is a sum of $T$ conditionally independent increments, we may use martingale concentration to find that, with probability at least $1-T^{-1}$, this sum is at most its mean (which features in the expected regret bound) plus a deviation of order $\sqrt{T \ln T}$. We obtain a high-probability analogue of Theorem~\ref{Thm:quantilebound} with slightly inflated constant.

\noindent
\textbf{Large-Feature-Vectors-as-Outliers} We may also deal with
non-i.i.d.\ gradients using exactly the same techniques developed above,
as follows. We assume that $f_t(\w) = h_t(\w^\top \X_t)$, where
$\X_t \in \mathbb R^d$ is a feature vector available at the beginning of
round $t$, and $h_t$ is a scalar Lipschitz convex loss function, revealed at the end of round $t$. This setting includes e.g.\ linear classification with hinge or logistic loss. Upon assuming that feature vectors $\X_1, \X_2, \ldots$ are drawn i.i.d.\ from $\pr$ (while the $h_t$ are arbitrary, possibly adversarially chosen), we can take the $p$-quantile $X_p \df q_p(\norm{\X}_*)$ of the dual norm of the feature vectors. We may then measure the robust expected regret \eqref{eq:exp.rob.reg} on the inlier rounds  $\cS = \set*{t \in [T] : \norm{\X_t}_* \le X_p}$, and obtain the analogue of Theorem~\ref{Thm:quantilebound}, where the only subtlety is using the gradient bound on $h_t$ to transfer from inlier $\X_t$ to small loss.

\begin{proposition}\label{prop:scalar.lipschitz.convex}
  Consider a joint distribution on sequences of feature vectors and scalar Lipschitz convex functions $(\X_1, h_1), (\X_2, h_2), \ldots$ such that the feature vectors $\X_1, \X_2, \ldots$ are i.i.d.\ with distribution $\pr$ on $\mathbb R^d$.\footnote{We do not constrain the distribution of $h_1,h_2,\ldots$, so we can model adversarial loss functions that are correlated with the feature vectors.} Let $X_p = q_p(\norm{\X}_*)$ be the $p$-quantile of the feature dual norm.
  Let ALG be an algorithm for online-convex optimisation over a domain
  of diameter $D$ and loss functions $f_t(\w) = h_t(\w^\top \X_t)$ that
  guarantees individual-sequence regret bounded by $B_T(X)$ in any
  $T$-round interaction with $\norm{\X_t}_* \le X$, without having to
  know $X$ up front. Consider FILTER
  Meta-Algorithm~\ref{alg:filtermetaAlgorithmQuantile} with $\grad_t$
  replaced by $\X_t$. Then the expected robust regret on inlier rounds
  $\cS = \set{t \in [T] : \norm{\X_t}_* \le X_p}$ is bounded by
  \[
    \bar R_T
    =
    \ex \sbr*{
      \max_{\u \in \domain}
      \sum_{t \in \cS} \del*{
        f_t(\w_t) - f_t(\u)
      }
    }
    \le
    B_{p T}(X_p)
    +
    D X_p \del*{
      4 \sqrt{2 p (1-p) T \ln T}
      + \frac{13}{3} (\ln T)^2
      + 3
    }
    .
  \]
\end{proposition}

\begin{proof}
  The proof follows that of Theorem~\ref{Thm:quantilebound}, with one extra (standard) step. Namely, to bound the loss on inlier rounds (for the dropped rounds term $P^{(2)}$ in the proof, and the concentration failure term $P^{(3)}$ in the proof), we use convexity, H\"older and bounded derivative to obtain
  \[
    f_t(\w_t) - f_t(\u^*)
    ~\le~
    h_t(\w_t^\top \X_t) - h_t({\u^*}^\top \X_t )
    ~\le~
    h_t'(\w_t^\top \X_t) (\w_t - \u^*)^\top \X_t
    ~\le~
    D X_p
    .
  \]
\end{proof}

\paragraph{Online-to-Batch Example}
We now discuss an example where the standard theory for stochastic gradient descent does not apply, but the iterate average of online gradient descent with quantile-based filtering still gives risk convergence guarantees. To keep things simple, we work in the one-dimensional setting with $\domain=[-1,+1]$. To stay within the assumptions of Proposition~\ref{prop:scalar.lipschitz.convex}, we take $f_t$ to be the logistic loss $f_t(w) = h_t(w \X_t)$ with $h_t(z) = \ln(1+e^{- y_t z})$ for $y_t \in \set{-1,+1}$. To make things interesting, we take $\X_t \in \mathbb R$ to have a distribution with heavy tails, with $\pr(\abs{\X_t} > x)$ of order $x^{-(1+\gamma)}$ for large enough $x$, for some $\gamma \in (0,1)$. Taking $\gamma > 0$ ensures that the expected loss $\ex[f_t(w)]$ is finite (as $f_t(w) \approx (-\X_t y_t w)_+$ for large $\X_t$), and hence has a bonafide minimiser (which can be in the interior or on the boundary, depending on the details of the distribution). Taking $\gamma < 1$ ensures that the tails are so heavy that $\ex[f_t'(w)^2] = \infty$ (as $f_t'(w) \approx  (-y_t \X_t)_+$ for large $\X_t$), and hence standard theory for SGD does not apply. Instead we will use Lemma~\ref{lem:huber} and Proposition~\ref{prop:scalar.lipschitz.convex} to argue that the filtered iterate average $\bar w_T$ approximates the minimiser of the risk $u^*$ in the sense that
\begin{equation}\label{eq:ogdworksb}
  \ex\nolimits_{\pr} \sbr*{\risk_{\pr}(\bar w_T) - \risk_{\pr}(u^*)}
  ~\to~ 0
  \quad
  \text{as}
  \quad
  T \to \infty
  .
\end{equation}
To bound the risks above, we will decompose $\pr = p P + (1-p) Q$ where $p$ is a quantile level chosen below, $P = \pr \delc[\big]{\cdot}{\norm{\X_t} \le X_p}$  and $Q = \pr \delc[\big]{\cdot}{\norm{\X_t} > X_p}$. We will bound the $\pr$-risks in terms of $P$-risks, then we will use Lemma~\ref{lem:huber} to  bound the $P$-risk difference in terms of the robust regret, and we will use Proposition~\ref{prop:scalar.lipschitz.convex} to bound that regret. We will settle on picking $p = 1-\frac{1}{\sqrt{T}}$. This has the effect that the $p$-quantile is $X_p \propto T^{\frac{1}{2(1+\gamma)}}$ (by inverting the tail probability). On the one hand, for any $w \in \domain$, the bias, i.e.\ the difference in risk on $P$ (inliers only) and on $\pr$ (full distribution), is at most of order
\begin{align*}
  \abs{
    \risk_{\pr}(w)
    -
    p \risk_{P}(w)
  }
  &~\le~
  \ex\nolimits_{\pr} \sbr*{ f_t(w) \mathbf 1_{\abs{\X_t} > X_p}}
  ~\approx~
    \ex\nolimits_{\pr} \sbr*{ (- \X_t y_t w) \mathbf 1_{\abs{\X_t} > X_p}}
  \\
  &~\le~
    \ex\nolimits_{\pr} \sbr*{ \abs{\X_t} \mathbf 1_{\abs{\X_t} > X_p}}
    ~=~
    \int_{X_p}^\infty \pr(\abs{\X_t} > x) \dif x
    ~\propto~
    X_p^{-\gamma}
    \propto T^{-\frac{\gamma}{2(1+\gamma)}}
    .
\end{align*}
On the other hand, the regret bound for $T$-round online gradient descent with gradient norms bounded by $X$ is $B_T(X) = O(D X \sqrt{T})$. Hence for our choice of $p$, the first term in the bound from Proposition~\ref{prop:scalar.lipschitz.convex} is dominant and of order $X_p \sqrt{T}$. Dividing by $T$ to plug in to Lemma~\ref{lem:huber} results in $\frac{X_p\sqrt{T}}{T} \propto T^{-\frac{\gamma}{2(1+\gamma)}}$. Both contributions (bias and regret) are of the same order and converge to zero, indicating that quantile-filtered online gradient descent achieves \eqref{eq:ogdworksb}.

\section{Conclusion and Future Work}
\label{sec:conclusion}

We have shown that the robust regret can be controlled for adversarial data
when there are at most $k$ outliers. A general question that we leave open is
whether it is possible to get a bound for adversarial losses that does not
depend on the number of outliers $k$, but on some other natural property of the
losses. For instance, we may try to incorporate prior knowledge about the size
of the gradients by specifying a prior $\pi$ on gradient norms and bounding the
robust regret in terms of the prior probability $\pi(G(\cS))$ of the size of
the inlier gradients. A possible way to approach this might be to introduce
specialist experts for different thresholds $G$ and then aggregate these. This
runs into severe difficulties, however, because we only find out whether a
specialist should be active or not in round~$t$ \emph{after} making our
prediction $\w_t$ and observing $\grad_t$. Moreover, specialists would have
different loss ranges and the robust regret can only depend on the loss range
$G(\cS)$ of the correct specialist. 

We also provided a sublinear bound on the robust regret for i.i.d.\ gradients
when the outliers are defined as rounds in which the gradients exceed their
$p$-quantile, or when they can be bounded in terms of an i.i.d.\ variable
$\X_t$. Alternatively, outliers might be defined as gradients with norms
exceeding their empirical $p$-quantile at the end of $T$ rounds. For i.i.d.\
gradients, the empirical $p$-quantile after $T$ rounds is close to the actual
$p$-quantile with high probability, so this case can be handled by running the
method from Section~\ref{sec:quantileMethod} for a slightly inflated $p$.
However, the empirical quantile formulation continues to make sense even when
gradients are not i.i.d., so it would be interesting to know whether a linear
number of outliers can be tolerated in any such non-i.i.d.\ cases.

\paragraph{Acknowledgments}

Van Erven and Sachs were supported by the Netherlands Organization
for Scientific Research (NWO) under grant number VI.Vidi.192.095. Kot{\l}owski was supported by the Polish National Science Centre under grant No.\ 2016/22/E/ST6/00299.

\bibliographystyle{abbrvnat}
\bibliography{../robust}

\appendix

\section{Proofs for Examples from Section~\ref{sec:okbound}}
\label{app:exampleproofs}

\begin{proof}[Proof of Corollary~\ref{cor:generalconvex}]
  Let $\passed \subset [T]$ denote the rounds that are passed on to OGD.
  Then the linearized regret of OGD on the rounds in $\passed$ is
  bounded by
  \[
    B_T(2 G(\cS))
      \leq 2 D \sqrt{\sum_{t \in \passed} \|\grad_t\|_2^2},
  \]
  as follows from arguments similar to those by
  \citet{JMLR:v12:duchi11a} (see e.g.\ Corollary~2 by
  \citet{OrabonaPal2018}). Bounding further, we obtain
  \begin{align*}
    B_T(2 G(\cS))
      &\leq 2 D \sqrt{\sum_{t \in \cS}
      \|\grad_t\|_2^2 + \sum_{t \in \passed \setminus \cS}
      \|\grad_t\|_2^2}
      \leq 2 D \sqrt{\sum_{t \in \cS}
      \|\grad_t\|_2^2} + 2 D \sqrt{\sum_{t \in \passed \setminus \cS}
      \|\grad_t\|_2^2}\\
      &\leq 2 D \sqrt{\sum_{t \in \cS} \|\grad_t\|_2^2}
           + 2 D G(\cS) \sqrt{k},
  \end{align*}
  where the last step uses that $|\passed \setminus \cS| \leq |[T]
  \setminus \cS| \leq k$ by assumption on $\cS$. Plugging this into
  \eqref{eqn:okbound} and bounding $R_T(\u,\cS) \leq \linReg_T(\u,\cS)$
  and $D(\u,\cS) \leq D$, the first inequality in
  \eqref{eqn:generalconvexlosses} follows. Finally, using that
  $\|\grad_t\|_2 \leq G(\cS)$ for all $t \in \cS$, we see that the second
  inequality holds as well.
\end{proof}

\begin{proof}[Proof of Corollary~\ref{cor:stronglyconvex}]
  Let $\passed \subset [T]$ denote the rounds that are passed on to ALG.
  By the proof of Theorems~2.1 and 4.1 of
  \citet{BartlettHazanRakhlin2007}, the linearized regret of ALG on the
  rounds in $\passed$ is bounded by
  \begin{equation*}
    \linReg_T(\u,\passed)
      \leq 
      \frac{1}{2} \sum_{t \in \passed} \frac{\|\grad_t\|_2^2}{\sigma t}
      + \frac{\sigma}{2} \sum_{t \in \passed} \|\w_t - \u\|_2^2
      \leq 
      \frac{2 G(\cS)^2}{\sigma} \big(\ln T + 1\big)
      + \frac{\sigma}{2} \sum_{t \in \passed} \|\w_t - \u\|_2^2.
  \end{equation*}
  Plugging this into Theorem~\ref{thm:okbound} and applying the
  definition of strong convexity, we get that the robust regret is
  bounded by
  \begin{align*}
    R_T(\u,\cS)
      &\leq \linReg_T(\u,\cS)
        - \frac{\sigma}{2} \sum_{t \in \cS} \|\w_t - \u\|_2^2\\
      &\leq 
      \frac{2 G(\cS)^2}{\sigma} \big(\ln T + 1\big)
      + 4D(\u,\cS) G(\cS) (k+1)
      + \frac{\sigma}{2} \sum_{t \in \passed} \|\w_t - \u\|_2^2
      - \frac{\sigma}{2} \sum_{t \in \cS} \|\w_t - \u\|_2^2\\
      &\leq 
      \frac{2 G(\cS)^2}{\sigma} \big(\ln T + 1\big)
      + 4D(\u,\cS) G(\cS) (k+1)
      + \frac{\sigma}{2} \sum_{t \in \passed \setminus \cS} \|\w_t -
      \u\|_2^2\\
      &\leq 
      \frac{2 G(\cS)^2}{\sigma} \big(\ln T + 1\big)
      + 4D(\u,\cS) G(\cS)(k+1)
      + \frac{\sigma D(\u,\cS)^2}{2} k.
  \end{align*}
  From this the desired result follows because $G(\cS) \leq
  \approxG(\u,\cS)/2$ and
  \[
    D(\u,\cS) \leq \max_{t : \|\grad_t\|_2 \leq 2 G(\cS)} \frac{\|\grad_t\|_2 +
      \|\nabla f_t(\u)\|_2}{\sigma} \leq \frac{\approxG(\u,\cS)}{\sigma}
  \]
  by Lemma~\ref{lem:stronglyconvexdomain} below.
\end{proof}

\begin{lemma}\label{lem:stronglyconvexdomain}
  Suppose $f_t$ is $\sigma$-strongly convex. Then
    $\|\w - \u\|_2 \leq \frac{\|\nabla f_t(\w)\|_2 + \|\nabla f_t(\u)\|_2}{\sigma}$
  for all $\w,\u \in \domain$.
\end{lemma}

\begin{proof}
  Applying the definition of $\sigma$-strong convexity twice, we have
  \[
   (\w - \u)^\top \nabla f_t(\u) + \frac{\sigma}{2} \|\w - \u\|_2^2
    \leq f_t(\w) - f_t(\u) \leq (\w - \u)^\top \nabla f_t(\w) -
    \frac{\sigma}{2} \|\u - \w\|_2^2,
  \]
  which leads to
  \begin{align*}
    \sigma \|\w - \u\|_2^2
      &\leq (\w - \u)^\top \big(\nabla f_t(\w) - \nabla f_t(\u)\big)\\
      &\leq \|\w - \u\|_2 \|\nabla f_t(\w) - \nabla f_t(\u)\|_2
      \leq \|\w - \u\|_2 \big(\|\nabla f_t(\w)\|_2 + \|\nabla
      f_t(\u)\|_2\big),
  \end{align*}
  from which the result follows.
\end{proof}

\begin{proof}[Proof of Lemma~\ref{lem:huber}]
  Let
  \[
    \trisk_{P_\epsilon}(\w) = \ex_{(M,\xi) \sim P_\epsilon} \sbr*{\tilde
    f(\w,M,\xi)}
  \]
  denote the risk for the modified loss under the mixture distribution
  $P_\epsilon$. Then the key to both results is to observe that
  \[
    \risk_P(\bar \w_T) - \risk_P(\u_P)
      = \frac{\trisk_{P_\epsilon}(\bar \w_T) - \trisk_{P_\epsilon}(\u_P)}{1-\epsilon}.
  \]
  This allows us to apply standard results for online-to-batch
  conversion to the modified losses $\tilde f$ under distribution
  $P_\epsilon$: the first inequality follows by combining
  \[
    \ex_{P_\epsilon} \sbr*{\trisk_{P_\epsilon}(\bar \w_T) - \trisk_{P_\epsilon}(\u_P)}
   \leq \frac{\ex_{P_\epsilon} \sbr*{\sum_{t=1}^T \del*{\tilde
   f(\w_t,M_t,\xi_t) - f(\u_P,M_t,\xi_t)}}}{T},
  \]
  with \eqref{eqn:robustasregularregret}, and the second result follows
  by applying Corollary~2 of \citet{CesaBianchiConconiGentile2004} to
  the modified excess losses $\tilde f(\w,M,\xi) - \tilde
  f(\u_P,M,\xi)$. By the boundedness assumption on the original excess
  loss a.s.\ under $P$, these are bounded in $[-B,B]$ a.s.\ under
  $P_\epsilon$, and we obtain
  \[
    \trisk_{P_\epsilon}(\bar \w_T) - \trisk_{P_\epsilon}(\u_P)
    \leq \frac{\sum_{t=1}^T \del*{\tilde
   f(\w_t,M_t,\xi_t) - \tilde f(\u_P,M_t,\xi_t)}}{T} + 2B \sqrt{\frac{2}{T} \ln
   \frac{1}{\delta}}
  \]
  with $P_\epsilon$-probability at least $1-\delta$. The second result
  then follows by plugging in \eqref{eqn:robustasregularregret} again.
\end{proof}

\begin{proof}[Proof of Corollary~\ref{cor:hubercor}]
  Let $\mathcal A$ be the event that \eqref{eqn:huberProb} holds with
  $B=DG$ and $\delta$ replaced by $\delta/2$, and let $\mathcal B$ be
  the event that the number of outliers $\sum_{t=1}^T M_t$ as at most
  $k$. Then Corollary~\ref{cor:generalconvex} implies that
  \eqref{eqn:huberapplied} holds on the intersection of $\mathcal A$ and
  $\mathcal B$, because
  \begin{align*}
    &\risk_P(\bar \w_T) - \risk_P(\u_P)
      \leq \frac{2 D G \Big(\sqrt{T} + 2k + \sqrt{k} + 2\Big)}{(1-\epsilon)T}
      + \frac{2DG}{1-\epsilon} \sqrt{\frac{2}{T} \ln \frac{2}{\delta}}\\
      &\leq \frac{2 D G \Big(\sqrt{T} + 3k + 2\Big)}{(1-\epsilon)T}
      + \frac{2DG}{1-\epsilon} \sqrt{\frac{2}{T} \ln \frac{2}{\delta}}\\
      &\leq
      \frac{6 D G \epsilon}{1-\epsilon}
      +
      \frac{2 D G\left(1 + 3 \sqrt{2 \epsilon (1-\epsilon)
      \ln(2/\delta)}\right)}{(1-\epsilon) \sqrt{T}}
      + \frac{2 D G \big((1-\epsilon)
  \ln(2/\delta)
       + 5\big)}{(1-\epsilon)T}
      + \frac{2DG}{1-\epsilon} \sqrt{\frac{2}{T} \ln \frac{2}{\delta}}
       \\
      &\leq 
      12 D G \epsilon
      +
\frac{2DG \big(2 + 5 
      \sqrt{2\ln(2/\delta)}\big)}{\sqrt{T}}
      + \frac{2 D G \big(
  \ln(2/\delta)
       + 10\big)}{T},
  \end{align*}
  where the last inequality uses the assumption that $\epsilon \leq 1/2$
  to obtain a simpler expression.
   Now the probability of $\mathcal
  A$ is at least $1-\delta/2$ by the second result of
  Lemma~\ref{lem:huber}, which applies because
  \begin{equation}\label{eqn:boundedexcessloss}
    -DG \leq -(\w-\u_P)^\top \nabla
    f(\u_P,\xi)   \leq f(\w,\xi) - f(\u_P,\xi) \leq (\w-\u_P)^\top \nabla
    f(\w,\xi) \leq DG
  \end{equation}
  $P$-almost surely. And the probability of $\mathcal B$ is at least
  $1-\delta/2$ by Bernstein's inequality
  \citep[Section~2.8]{BoucheronLugosiMassart2013}. Hence, by the union
  bound, it follows that both $\mathcal A$ and $\mathcal B$ hold
  simultaneously with probability at least $1-\delta$, as required.
\end{proof}

\section{Proof of Theorem~\ref{Thm:quantilebound} from Section~\ref{sec:quantileMethod}}
\label{appx:proof.quantiles}
\begin{proof}
The point of departure is that, by definition, $\ind\set*{\norm{\grad_t}_* \le G_p}$ is i.i.d.\ Bernoulli-$p$. By Bernstein's concentration inequality (for fixed time $t$), we have that w.p.\ $\ge 1-\delta$
\[
  \frac{1}{t} \sum_{s=1}^t \ind\set*{\norm{\grad_s}_* \le G_p}
  ~\ge~
  p
  - \sqrt{\frac{2 p(1-p) \ln \frac{1}{\delta}}{t}}
  - \frac{\ln \frac{1}{\delta}}{3 t}
\]
Rephrasing this event with the empirical quantile function $\hat q_t$, we see that w.p.\ $\ge 1-\delta$,
\begin{equation}\label{eq:concentration}
  G_p
  ~\ge~
  \LCB_t
  ~\df~
  \hat q_t\del*{  p
    - u_t
  }
  \qquad
  \text{where}
  \qquad
  u_t
  \df
  \sqrt{\frac{2 p(1-p) \ln \frac{1}{\delta}}{t}}
  + \frac{\ln \frac{1}{\delta}}{3 t}
  .
\end{equation}
We apply the analogous concentration in the other direction to find that with probability at least $1-\delta$,
\begin{equation}\label{eq:concentration2}
  G_{p - 2 u_t - 3/(2 t) \ln \frac{1}{\delta}} \le \LCB_t
\end{equation}
where the extra margin $\frac{3 \ln \frac{1}{\delta}}{2 t}$ is necessary to correct for the fact that $p - 2 u_t$ may be closer to $1/2$ than $p$, and hence require a slightly enlarged confidence width.
In the remainder of the proof, we fix $\delta = T^{-2}$, to ensure that the probability of failure of either event \eqref{eq:concentration} or \eqref{eq:concentration2} over the course of $T$ rounds is at most $\frac{2}{T}$.

Let $\mathcal E_t$ denote the event that \eqref{eq:concentration} and \eqref{eq:concentration2} hold at round $t$, and let $\mathcal E = \bigcap_{t=1}^T \mathcal E_t$. We split the expected robust regret in three parts $\bar R_T \le P^{(1)} + P^{(2)} + P^{(3)}$ spelled out below, depending on whether the desired concentration event $\mathcal E$ holds or not, and within $\mathcal E$ we split the rounds in those where FILTER passes the gradient on to ALG and those were it was ignored:
\begin{align*}
  P^{(1)}
  &~\df~
  \ex \sbr*{
    \ind_{\mathcal E}
    \max_{\u \in \domain} \sum_{t \in [T] : \norm{\grad_t}_* \le \LCB_{t-1}} \tuple{\w_t - \u, \grad_t}
    }
  \\
  P^{(2)}
  &~\df~
  \ex \sbr*{
    \ind_{\mathcal E}
    \max_{\u \in \domain} \sum_{t \in [T] : \LCB_{t-1} < \norm{\grad_t}_* \le G_p} \tuple{\w_t - \u, \grad_t}
    }
  \\
  P^{(3)}
  &~\df~
  \ex \sbr*{
    \ind_{\mathcal E^c}
    \max_{\u \in \domain} \sum_{t \in [T] : \norm{\grad_t}_* \le G_p} \tuple{\w_t - \u, \grad_t}
  }
\end{align*}
For part $P^{(1)}$, we apply the individual-sequence regret bound $B_{\hat T}(G_p)$ of ALG, where $\hat T$ is the random number of rounds, for which we have $\ex\sbr{\hat T} \le p T$. We then drop the indicator and Jensen the expectation inside to get $P^{(1)} \le B_{p T}(G_p)$. For part $P^{(2)}$, we use that $\mathcal E$ and $\LCB_{t-1} < \norm{\grad_t}_*$ imply that $G_{p - 2 u_t - 3/(2 t) \ln \frac{1}{\delta}} < \norm{\grad_t}_* < G_p$ to find
\[
  \ind_{\mathcal E}
  \max_{\u \in \domain}
\sum_{t \in [T] : \LCB_{t-1} < \norm{\grad_t}_* \le G_p} \tuple{\w_t - \u, \grad_t}
~\le~
\ind_{\mathcal E}
\sum_{t \in [T] : G_{p - 2 u_t - 3/(2 t) \ln \frac{1}{\delta}} < \norm{\grad_t}_* \le G_p} D G_p
\]
Dropping the indicator, taking expectation and using the definition of quantile yields
\begin{align*}
  P^{(2)}
  &~\le~
  D G_p
  \sum_{t=1}^T
    \pr\set*{
    G_{p - 2 u_t - 3/(2 t) \ln \frac{1}{\delta}} < \norm{\grad_t}_* < G_p
    }
  \\
  &~\le~
    D G_p \sum_{t=1}^T (2 u_t + 3/(2 t) \ln \frac{1}{\delta})
  \\
  &~\le~
    D G_p \del*{
    4 \sqrt{2 p (1-p) T \ln T}
    + \frac{13}{3} (\ln T)^2 + 1
  }
\end{align*}
Finally, for $P^{(3)}$ we use that the integrand is bounded by $T D G_p$, and the error probability by $\pr(\mathcal E^c) \le \frac{2}{T}$ to find $P^{(3)} \le 2 D G_p$.
\end{proof}

\end{document}